\documentclass[twoside]{article}

\usepackage[preprint]{arxiv_style}
\usepackage[utf8]{inputenc}
\usepackage{nicefrac}

\usepackage[round]{natbib}

\usepackage{balance}
\usepackage{multirow}
\usepackage{cite}
\usepackage{amsmath}
\usepackage{amsthm}
\usepackage{amssymb}
\usepackage{blkarray}
\usepackage{booktabs}
\usepackage{xspace}
\usepackage{graphicx}
\usepackage{subfigure}
\usepackage{enumitem}
\usepackage{url}
\usepackage{algorithm}
\usepackage[noend]{algorithmic}
\usepackage[hidelinks]{hyperref}
\newtheorem{theorem}{Theorem}
\newtheorem{definition}[theorem]{Definition}
\newtheorem{problem}[theorem]{Problem}

\DeclareMathOperator*{\avg}{\mathrm{avg}}

\newcommand{\ex}{\mathrm{E}}

\newcommand{\ViolatedHyperplane}{\texttt{Viol-HP}\xspace}
\newcommand{\wfhs}{\mathrm{wfhs}}
\newcommand{\ipalg}{{\sc RippleK}\xspace} %
\newcommand{\halg}{{\sc RumrunnerK}\xspace} %

\newif\ifallproofs

\newcommand{\Sushi}{\texttt{Sushi}\xspace}
\newcommand{\Afive}{\texttt{A5}\xspace}
\newcommand{\Anine}{\texttt{A9}\xspace}
\newcommand{\Aseventeen}{\texttt{A17}\xspace}
\newcommand{\Afortyeight}{\texttt{A48}\xspace}
\newcommand{\Aeightyone}{\texttt{A81}\xspace}
\newcommand{\SFwork}{\texttt{SFwork}\xspace}
\newcommand{\SFshop}{\texttt{SFshop}\xspace}

\allowdisplaybreaks

\begin{document}

\twocolumn[

\aistatstitle{Approximating a RUM from Distributions on $k$-Slates}

\aistatsauthor{Flavio Chierichetti \And Mirko Giacchini \And Ravi Kumar \And Alessandro Panconesi \And Andrew Tomkins }

\aistatsaddress{ Sapienza University \\ Dip. di Informatica \And Sapienza University \\ Dip. di Informatica  \And Google \\ Mountain View, CA \And Sapienza University \\ Dip. di Informatica\And Google \\ Mountain View, CA \\
} ]

\begin{abstract}
In this work we consider the problem of fitting Random Utility Models (RUMs) to user choices.  Given the winner distributions of the subsets of size  $k$ of a universe, we obtain a polynomial-time algorithm that finds the RUM that best approximates the given distribution on average.  Our algorithm is based on a linear program that we solve using the ellipsoid method.  Given that its corresponding separation oracle problem is NP-hard,  we devise an approximate separation oracle that can be viewed as a generalization of the weighted feedback arc set problem to hypergraphs. Our theoretical result can also be made practical: we obtain a heuristic that is effective and scales to real-world datasets.
\end{abstract}

\section{INTRODUCTION}

In this paper we consider the setting of discrete choice, in which users must choose an item from a set of alternatives. The gold standard for highly representative models of discrete choice is the family of \emph{Random Utility Models}, or \emph{RUMs}, which are capable of encoding a range of complex outcomes under a broad notion of rational user behavior. RUMs are quite  powerful but unfortunately there are no known algorithms guaranteed to learn or approximate a RUM from samples of user choices.

It is therefore of great interest to find special cases for which tractable algorithms are possible. Recent work of~\citet{ackpt22} showed that efficient algorithms for approximate learning of general RUMs are possible in the \emph{pairwise} setting---if the user choices being modeled contain only two options.

There are some settings in which pairwise choices arise naturally, such as head-to-head testing of two items from a universe, or two-player games. But most practical settings involve choices among larger sets of options. Many online platforms offer interfaces with multiple options to choose from: ten blue links, or a fixed set of movie options shown in a carousel. In this paper, we extend earlier results to give a polynomial-time algorithm to approximately learn RUMs on slates of size at most $k$, for any constant $k$. Formal definitions for approximate learning of RUMs on slates of size at most $k$ are given in Section~\ref{sec:prelim}.

To describe the motivation and the idea behind the algorithm, we must first introduce the idea of RUMs as distributions over permutations over a universe of items $[n]$.%
When a user is presented with a slate of objects, a random permutation is drawn, and the object selected is the one of highest rank in the permutation. Note that a RUM, for each slate $S$, induces a probability distribution %
over $S$, specifying the probability that $i \in S$ is selected by the user when $S$ is given. 
In this context, there are two natural learning tasks to consider. The input to both tasks consists of a set of pairs $(i,S)$ where $i \in S$ is the object selected by the user when $S$ is presented. With this information one can compute the empirical probability that an object is selected from a given slate. The first learning task is to compute a RUM whose induced probability distributions best approximate the empirical %
distributions. The second learning task is, given a set of past such interactions, to predict the empirical probability in the future. 

\paragraph{Our Contributions.} In this paper we give, for both learning tasks, {\em (i)} a polynomial-time algorithm for slates of size $k \ge 2$, thereby  significantly extending the results in~\citet{ackpt22}, and {\em (ii)}  fast heuristics that are effective in practice. 

As in~\citet{ackpt22}, we use the celebrated ellipsoid method for linear programming (LP), together with a separation oracle. Unfortunately, their approach only works only for  the case of slates of size $2$, i.e., for $k = 2$. %
The ellipsoid method, as it is well-known, allows to cope with LPs with exponentially many constraints, provided that a so-called separation oracle is available.
The technical difficulty here is to exhibit such a separation oracle for the case of slates of size $k>2$. In particular, the separation oracle can no longer be
formed by solving the minimum \emph{feedback arc set} problem (FAS) as in~\citet{ackpt22}. Instead, we require a solution to a new problem that represents a hypergraph extension of FAS. The main contribution of this work
is defining the new extended problem, giving a polynomial-time algorithm that produces an approximate solution, and
then showing that the resulting approximate separation oracle provides a sufficient approximation to the best possible RUM.

To conclude, we would like to mention a couple of positive features of our approach. Our polynomial-time solution requires information about all slates of dimension $k$ in order to work, but such complete datasets might not be available in practice. We thus provide a heuristic that is able to cope effectively with such contingencies. We also provide heuristic but practically more efficient separation oracles, which  are described in the experimental section.

\paragraph{Previous Work.}
Discrete choice theory is a well-established research topic in machine learning and economics; see~\citep{train} for an excellent introduction and~\citep{ckt18b,ckt21,ros20,sru20,spu19} for some recent work.
RUMs are a general class that contains, e.g.,  Multinomial Logits (or MNLs) models and Mixed MNLs~\citep{mt00,ckt21}, as special cases.
RUMs are equivalent~\citep{ckt18,fjs09} to models in which a user samples an ordering of the objects in the universe, and given a slate, selects the object in the slate ranked the highest in the ordering.

RUMs, and subclasses of RUMs, have also been extensively studied from both active learning and passive learning perspectives~\citep{spx12,os14,ckt18,ckt18b,notx18,tang20}. The problems of efficiently representing and sketching RUMs have been studied~\citep{fjs09,ckt21,ackpt22}.

The work that is closest to ours is \citep{ackpt22}, which also aims to learn RUMs from choices on slates by solving an LP via the ellipsoid method with an approximate separation oracle.  The differences between our work and theirs can be summarized as follows: (i) their algorithm  works {\em only} for slates of size $k=2$, while ours works for any constant $k$, (ii) their LP also works \emph{only} for $k= 2$ and, most importantly, (iii) our separation oracles are \emph{different} and \emph{new}: we (approximately) solve a novel generalization of the FAS problem; this might be of independent interest.

\paragraph{Organization.} Section~\ref{sec:prelim} establishes the  notation.  Section~\ref{sec:lp} provides formulations of the primal and dual LPs for RUM fitting, Section~\ref{sec:wfhs_alg} introduces the pivotal element of the approximate separation oracle for the LP, Section~\ref{sec:poly_alg} shows that this oracle can be used with the ellipsoid method to find an approximately optimal RUM. Section~\ref{sec:nalb} discusses lower bounds for the size of the input for our LP method.  Section~\ref{sec:exp} presents our experimental results.

\section{PRELIMINARIES
} \label{sec:prelim}

For a set $S$, let $2^S$ denote the set of all subsets of $S$ and let ${S \choose k}$ denote the set of all $k$-sized subsets of $S$.  For a distribution $D$, let $x \sim D$ denote that the random variable $x$ is drawn according to $D$ and let $D(i)$ denote $\Pr_{x \sim D}[x = i]$, where $i$ is in the support of $D$.

Let $[n]$ denote $\{1,\ldots,n\}$ and let $\mathbf{S}_n$ denote the set of permutations of $[n]$. 
A \emph{slate} is a non-empty subset of $[n]$.  For a slate $\varnothing \neq S \in 2^{[n]}$ and a permutation $\pi \in \mathbf{S}_n$, let $\pi(S)$ be the element of $S$ that ranks the highest in $\pi$.    

\begin{definition}[Random utility model (RUM)]
A \emph{random utility model (RUM)} $R$ on $[n]$ is a distribution $D$ on $\mathbf{S}_n$.  For a slate $S$, let $R_S$ denote the distribution of the random variable $\pi(S)$, where $\pi \sim D$.  We say that $R_S$ is the \emph{winner distribution} on $S$ induced by RUM $R$.
\end{definition}
Our goal is to fit RUMs to datasets in order to minimize the average $\ell_1$-error over the winner distributions.
\begin{definition}[Average RUM  approximation]
Let $\mathcal{S}$ be a set of slates of $[n]$ and for each $S \in \mathcal{S}$, let $\mathcal{P}(S)$ be a probability distribution over $S$.  We say that $\mathcal{P}$ can be
\emph{approximated on average} to within $\epsilon$ by a RUM $R$ if
$\avg_{S \in \mathcal{S}}\left|R_{S} - \mathcal{P}(S)\right|_1 \le \epsilon$.
Given $\mathcal{P}$, let $\epsilon_{1}(\mathcal{P})$ be the smallest\footnote{\label{footnote:min_vs_inf}This minimum exists since it is the optimal value of  a finite-sized, feasible, LP.} $x \ge 0$ such that there is a RUM that approximates $\mathcal{P}$ on average to within $x$.
\end{definition}
Recall that the $\ell_1$-distance between two distributions over $S$ is (exactly) twice as large as the total variation distance between them, i.e., it is twice as large as the maximum gap in the probabilities of an event in the two distributions.
\begin{problem}[Average RUM additive approximation] 
\label{prob:avg}
Given $\mathcal{S}$
and a corresponding $\mathcal{P}$, find a \emph{$\delta$-additive approximation} to $\epsilon_{1}(\mathcal{P})$, i.e., 
obtain a RUM whose average distance from $\mathcal{P}$ is not larger than $\epsilon_1(\mathcal{P}) + \delta$.
\end{problem}

We also need the following generalization of the weighted feedback edge set problem to hypergraphs.  For a  hyperedge $e \in {[n] \choose k}$ and a permutation $\pi \in \mathbf{S}_n$, let
$\pi(e)$ be the element of $e$ that ranks highest in $\pi$.
\begin{problem}[Weighted feedback hyperedge set (WFHS)]
\label{prob:wfhs}
An instance of the \emph{$(\tau,k)$-bounded weighted  feedback hyperedge set (WFHS)} problem is composed of a set $V = [n]$ of vertices, a set $E \subseteq \binom{V}k$ of hyperedges  and, for each $e \in E$, a non-negative weight function $w_e: e \rightarrow [0, \tau]$. The cost of a permutation $\pi$ of the vertices of $V$ is equal to $C(\pi) = \sum_{e \in E} w_e(\pi(e))$. The WFHS problem is to find a permutation $\pi \in \mathbf{S}_n$ that minimizes $C(\pi)$.
\end{problem}

\section{FITTING RUMS WITH LINEAR PROGRAMS}\label{sec:lp}
Our goal in this section is to write down a linear program (LP) whose solution gives the desired RUM. This LP will have exponentially many constraints but polynomially many variables. Thanks to the general theory of the ellipsoid method, it is possible to solve the LP to within a very small error, provided that a so-called \emph{separation oracle} exists. Such an oracle is provided in the next sections.

Let us begin by writing down an LP whose solution gives a RUM attaining minimum average $\ell_1$-error for the slates in $\mathcal{S}$, and which generalizes the LP of~\citet{ackpt22} to $k \ge 2$.\footnote{For the interested reader, the LP of~\citet{ackpt22} can be obtained from ours by setting $k = 2$. The main technical difference between the LPs stems from the fact that, when $k = 2$,     enforcing a bound on the error of the probability that a particular element wins in a slate is {\em equivalent} to enforcing a bound on the total variation error of the slate distributions  (for distributions $P = (x, 1-x)$ and $Q = (y, 1-y)$ on two elements, it is easy to see that $|P-Q|_{1} = 2 |P-Q|_\infty$); unfortunately, this property fails to hold if $k>2$. We thus have to provide an argument that makes use of a novel polyhedron ($\overline{F}_{\rho}$).}
In order to write it down, we need to have access to $\mathcal{P}$.  For simplicity, let $D_S$ denote $\mathcal{P}(S)$.  Then $D_S(i)$ is the (empirical) probability that $i$ wins in $S$, for each $i \in S \in \mathcal{S}$. (Observe that, for each $S \in \mathcal{S}$, $\sum_{i \in S} D_S(i) = 1$.) The LP assigns probability $p_\pi \ge 0$ to each permutation $\pi \in \mathbf{S}_n$, and requires that $\sum_{\pi \in \mathbf{S}_n} p_{\pi} = 1$.
\begin{equation}\label{eqn:prelp}\left\{\begin{array}{lcl}
\min \frac{1}{|\mathcal{S}|}\cdot \sum\limits_{S \in \mathcal{S}}\sum\limits_{i \in S} \epsilon_{S,i} &\\
 \epsilon_{S,i} + \hspace{-0.25cm}\sum\limits_{\substack{\pi\in \mathbf{S}_n\\\pi(S) = i}}\hspace{-0.25cm} p_{\pi} \ge D_S(i) & (L_{S,i}) & \forall i\hspace{-0.065cm} \in\hspace{-0.065cm} S \hspace{-0.065cm}\in\hspace{-0.065cm} \mathcal{S}\\
 \epsilon_{S,i} -\hspace{-0.25cm}\sum\limits_{\substack{\pi\in \mathbf{S}_n\\\pi(S) = i}} \hspace{-0.25cm}p_{\pi} \ge -D_S(i) & (U_{S,i}) & \forall i\hspace{-0.065cm} \in\hspace{-0.065cm} S \hspace{-0.065cm}\in\hspace{-0.065cm} \mathcal{S}\\
 \sum\limits_{\pi\in \mathbf{S}_n} p_{\pi} = 1 & (D) &\\
p_{\pi} \ge 0 & &\forall \pi\in \mathbf{S}_n\\
\epsilon_{S,i} \ge 0 && \forall i\hspace{-0.05cm} \in\hspace{-0.05cm} S \hspace{-0.05cm}\in\hspace{-0.05cm} \mathcal{S}
\end{array}\right.
\end{equation}
 Then, $\{p_{\pi}\}_{\pi \in \mathbf{S}_n}$ defines a RUM. Given any $i \in S \in \mathcal{S}$, it also requires that the approximation error made by every RUM that is a feasible solution for the pair $(i, S)$ is no greater than $\epsilon_{S,i}$. The constraints of type $L_{S,i}$ guarantee that the probability that $i$ wins in $S$ is at least $D_S(i) - \epsilon_{S,i}$; those of type $U_{S,i}$ guarantee that the same probability is at most $D_S(i) + \epsilon_{S,i}$. Therefore, the $\ell_1$-error made by the optimal RUM on slate $S$ is not more than $\sum_{i \in S} \epsilon_{S,i}$. The LP minimizes the average (i.e., the scaled sum) of the $\ell_1$-errors over all slates $\mathcal{S}$ (i.e., $|\mathcal{S}|^{-1}\sum\limits_{S \in \mathcal{S}}\sum\limits_{i \in S} \epsilon_{S,i}$). It follows that the optimal solution ensures that each $\epsilon_{S,i}$ {\em equals} the $\ell_1$-error on $S$ and therefore that its average $\ell_1$-error is the minimum achievable by any RUM.

LP~(\ref{eqn:prelp}) has exponentially many variables and polynomially many constraints. If we take its dual, we obtain an LP with polynomially many variables and exponentially many constraints.  Hence, it can be optimized efficiently, given a separation oracle. Here is the dual of~(\ref{eqn:prelp}):
$$\left\{\begin{array}{lcl}
\multicolumn{3}{l}{\max D + \sum\limits_{i \in S \in \mathcal{S}} \left(D_S(i) \cdot (L_{S,i} - U_{S,i})\right)} \\
L_{S,i} + U_{S,i} \le |\mathcal{S}|^{-1} & (\epsilon_{S,i})&\forall i \in S \in \mathcal{S}\\
D + \sum\limits_{S \in \mathcal{S}} (L_{S,\pi(S)} - U_{S,\pi(S)}) \le 0 & (p_{\pi}) & \forall \pi\in \mathbf{S}_n\\
L_{S,i},U_{S,i} \ge 0&&\\
D \text{ unrestricted}&&
\end{array}\right.$$

Observe that every feasible dual solution  can be transformed into a feasible solution with the same value and with the additional property that, for all $i \in S \in \mathcal{S}$, at least one of $L_{S,i}$ and $U_{S,i}$ is equal to zero.  To see this, observe that if the two variables are positive we can subtract $\min(L_{S,i},U_{S,i})$ from both without affecting feasibility and without changing the objective function's value. Given a feasible solution to the dual, define $\Delta_{S,i} = U_{S,i} - L_{S,i}$. With this transformation, we have $U_{S,i} = \max\left(\Delta_{S,i},0\right)$, $L_{S,i} = \max\left(-\Delta_{S,i},0\right)$, and $|\Delta_{S,i}| = L_{S,i} + U_{S,i}$. 
The dual LP is then equivalent to the following LP:
\begin{equation}\label{eqn:dual}\left\{\begin{array}{ll}
\max D - \sum\limits_{i \in S \in \mathcal{S}} \left(D_S(i) \cdot \Delta_{S,i}\right) &\\
\sum\limits_{S \in \mathcal{S}} \Delta_{S,\pi(S)} \ge D &  \forall \pi\in \mathbf{S}_n\\
-|\mathcal{S}|^{-1}\le \Delta_{S,i} \le |\mathcal{S}|^{-1} & \forall i \in S \in \mathcal{S}\\
D \text{ unrestricted}&
\end{array}\right.
\end{equation}
We now transform LP~(\ref{eqn:dual}) from a maximization problem into a feasibility problem, to pave the way for the ellipsoid algorithm:
\begin{equation*}%
\small
F_{\rho} := 
\left\{\begin{array}{@{\;}r@{\;}l@{\;\;\;}l}
&\multicolumn{2}{l}{\hspace{-0.17cm}D - \sum\limits_{i \in S \in \mathcal{S}} \left(D_S(i) \cdot \Delta_{S,i}\right) \ge \rho}\\
&\sum\limits_{S \in \mathcal{S}} \Delta_{S,\pi(S)} \ge D & \mbox{ \footnotesize  $\forall \pi\in \mathbf{S}_n$}\\
&-|\mathcal{S}|^{-1}\hspace{-0.05cm} \le \Delta_{S,i} \le \hspace{-0.05cm}|\mathcal{S}|^{-1} &\mbox{ \footnotesize  $\forall i \in S \in \mathcal{S}$}
\end{array}\right.
\end{equation*}
Observe that LP~(\ref{eqn:dual}) has value at least $\rho$ iff $F_{\rho}$ is feasible. In order to have the polytope lie in the non-negative orthant, we set $\overline{\Delta}_{S,i} = \Delta_{S,i} + |\mathcal{S}|^{-1}$. We then rewrite the expressions in terms of $\overline{\Delta}_{S,i}$:
\begin{align*}
&D - \sum\limits_{i \in S \in \mathcal{S}} \left(D_S(i) \cdot \Delta_{S,i}\right)\\
&= D - \sum\limits_{i \in S \in \mathcal{S}} \left(D_S(i) \cdot \left(\overline{\Delta}_{S,i} - |\mathcal{S}|^{-1}\right)\right) \\
&= D - \sum\limits_{i \in S \in \mathcal{S}} \left(D_S(i) \cdot \overline{\Delta}_{S,i} \right) + \sum\limits_{S \in \mathcal{S}} \sum\limits_{i \in S} \left(D_S(i) \cdot |\mathcal{S}|^{-1}\right)\\
&= D - \sum\limits_{i \in S \in \mathcal{S}} \left(D_S(i) \cdot \overline{\Delta}_{S,i} \right) + \sum\limits_{S \in \mathcal{S}}  |\mathcal{S}|^{-1}\\
&= D - \sum\limits_{i \in S \in \mathcal{S}} \left(D_S(i) \cdot \overline{\Delta}_{S,i} \right) + 1.
\end{align*}
Thus, $D - \sum\limits_{i \in S \in \mathcal{S}} \left(D_S(i) \cdot \Delta_{S,i}\right) \ge \rho$ is equivalent to $D - \sum\limits_{i \in S \in \mathcal{S}} \left(D_S(i) \cdot \overline{\Delta}_{S,i} \right) \ge \rho - 1$. Moreover, for $\pi \in \mathbf{S}_n$, $\sum_{S \in \mathcal{S}} \Delta_{S,\pi(S)} = \sum_{S \in \mathcal{S}} \left(\overline{\Delta}_{S,\pi(S)} - \left|\mathcal{S}\right|^{-1}\right) = \sum_{S \in \mathcal{S}}\overline{\Delta}_{S,\pi(S)} - 1$. Thus, $\sum_{S \in \mathcal{S}} \Delta_{S,\pi(S)} \ge D$ is equivalent to $\sum_{S \in \mathcal{S}} \overline{\Delta}_{S,\pi(S)} \ge D + 1$. Finally, $-|\mathcal{S}|^{-1} \le \Delta_{S,i} \le|\mathcal{S}|^{-1}$ is equivalent to $0 \le \overline{\Delta}_{S,i} \le 2 |\mathcal{S}|^{-1}$. Hence, the system $F_{\rho}$ is equivalent to the following system $\overline{F}_{\rho}$. 
\begin{equation}\label{eqn:feas}
\small
\overline{F}_{\rho} := 
\left\{\begin{array}{@{\;}r@{\;}l@{\;\;\;}l}
\mbox{\footnotesize$c_{\rho}:$}&\multicolumn{2}{l}{\hspace{-0.17cm}D - \sum\limits_{i \in S \in \mathcal{S}} \left(D_S(i) \cdot \overline{\Delta}_{S,i}\right) \ge \rho - 1}\\
\mbox{\footnotesize  $c_{\pi}:$}&\sum\limits_{S \in \mathcal{S}} \overline{\Delta}_{S,\pi(S)} \ge D + 1& \mbox{ \footnotesize  $\forall \pi\in \mathbf{S}_n$}\\
\mbox{\footnotesize  $c_{S,i}:$}&0 \le \overline{\Delta}_{S,i} \le \hspace{-0.05cm}2|\mathcal{S}|^{-1} &\mbox{ \footnotesize  $\forall i \in S \in \mathcal{S}$}
\end{array}\right.
\end{equation}
LP~(\ref{eqn:dual}) has value at least $\rho$ if and only if $\overline{F}_{\rho}$ is feasible.

To summarize, we have reduced Problem~\ref{prob:avg} to the feasibility problem $\overline{F}_{\rho}$. %
Since $\overline{F}_{\rho}$ has exponentially many constraints, we cannot check the validity of its $c_{\pi}$ constraints one by one\footnote{The number of $C_{S,i}$'s constraints and $c_{\rho}$ is only $O\left(\left|\mathcal{S}\right|\right)$, i.e., it is (sub)linear in the input size, and hence can be easily checked.}. On the other hand, the feasibility of the full set of the $c_{\pi}$ constraints of $\overline{F}_{\rho}$ can be checked by solving an instance of Problem~\ref{prob:wfhs}, which is an instance of the WFHS problem with weights given by the  $\overline{\Delta}_{S,i}$'s. In the next section, we provide an approximation algorithm for WFHS. And in the subsequent section, we use this approximation algorithm to provide an approximate separation oracle for $\overline{F}_{\rho}$, which we will then use to solve Problem~\ref{prob:avg}, our original RUM fitting problem, in polynomial time.

\section{AN ALGORITHM FOR THE WFHS PROBLEM}\label{sec:wfhs_alg}
We introduce notation and recall some results before discussing our  $O(\epsilon \cdot \tau \cdot n^{k})$-additive approximation algorithm for the WFHS problem; this algorithm, like the algorithms of~\citet{ks07,schudy12,fk99} for the feedback arc set (FAS) problem, makes use of tensor-maximization algorithms for $k$-CSP (constraint satisfaction problem) as a building block. 

Our $k$-CSP will have a constant-sized alphabet $[t]$, for $t = O(\epsilon^{-1})$. The $k$-CSP has one type of predicate $P$ of arity $k$: $P(x_1,\ldots,x_k) = \left[x_1 > \max(x_2,\ldots,x_k)\right]$.

An instance $\mathcal{I}$ of this \emph{$k$-CSP} is composed of a set $X = \{x_1,\ldots,x_n\}$ of variables, of a set $M$  of $k$-tuples of variables of $X$ (the constraints induced by $P$), and of a weighting $w: M \rightarrow [0,\tau]$. The variables take values over $[t]$. %

The goal of the problem is to assign values to the variables (where the value of $x_i \in X$ is an element of $[t]$) in order to maximize the total weight of $k$-tuples $(x_{i_1},\ldots,x_{i_k}) \in M$ that make the predicate $P(x_{i_1},\ldots,x_{i_k})$ true (i.e., maximize $\sum_{(x_{i_1},\ldots,x_{i_k}) \in M} \left(P(x_{i_1},\ldots,x_{i_k}) \cdot w(x_{i_1},\ldots,x_{i_k})\right)$). This maximum value is called the \emph{optimal value} of the $k$-CSP  instance and is denoted $\text{OPT}(\mathcal{I})$.

We will make use of the following result, that has been proved by various authors including, e.g., \citep{schudy12,y14}.
\begin{theorem}\label{thm:schudy}
For each constant $\epsilon > 0$, for each non-negative integers $t,k \ge 1$, and for each $k$-CSP over the alphabet $[t]$, there exists an algorithm that, in time $O(n^k)$, returns an assignment to the $n$ variables of a generic instance $\mathcal{I}$  of the $k$-CSP such that the expected weight of the $k$-tuples satisfied by the assignment is at least $\text{OPT}(\mathcal{I}) -\epsilon \cdot \tau \cdot  n^k$.
\end{theorem}
We also mention that a slower algorithm, with a runtime $n^{2^{O(k)}}$, for approximately solving  generic $k$-CSPs was given in~\citet{yz14}.

We now show that the WFHS problem can be additively approximated. Our approach, like that of~\citet{ks07,schudy12,fk99} for the feedback arc set (FAS) problem, solves the WFHS problem by first (i) casting it as a tensor-maximization/$k$-CSP problem and then (ii) transforming the solution to this problem into an ordering of the vertices.
\begin{theorem}\label{thm:wfhs}
For each constants $0 < \epsilon < \frac12, \alpha > 0, k \ge 2$, there exists an algorithm for the $(\tau,k)$-bounded WFHS problem  that can return an $(\epsilon \cdot \tau \cdot n^k)$-additive approximation, with probability at least $1- n^{-\alpha}$, in time $O( \alpha n^k \log n)$.
\end{theorem}
\begin{proof}
Given an instance $V,E,\{w_e\}_{e \in E}$ of the WFHS problem, we define a $k$-CSP with variables $x_1,\ldots,x_n$ over the alphabet $[t]$, for $t = \lceil \epsilon^{-1}\rceil$. For each hyperedge $e = \left\{i_1,\ldots,i_{k}\right\} \in E$, we create $k$ constraints: $C_{i_1,e} := P(x_{i_1},x_{i_2},\ldots,x_{i_{|e|-1}}, x_{i_{k}})$, $C_{i_2,e} := P(x_{i_2},x_{i_3},\ldots,x_{i_{k-1}}, x_{i_{k}},x_{i_1})$, $\ldots$, $C_{i_{k},e} := P(x_{i_{k}},x_{i_1},\ldots,x_{i_{k-2}}, x_{i_{k-1}})$. The weight of the constraint $C_{i_j,e}$ will be $w\left(C_{i_j,e}\right) = \tau - w_e(i_j)$ and let $\Gamma_e = \{C_{i_j,e} \mid i_j \in e\}$. Let $\mathcal{I}$ be the resulting $k$-CSP instance.

Now, let $\pi \in \mathbf{S}_n$. For each $i \in \left[t\right]$, and for each item $j$ having rank in $\pi$ in the set $R_i = \left[\left\lceil i \cdot \frac n t\right\rceil\right] \setminus \left[\left\lceil (i-1) \cdot \frac{n}t\right\rceil\right]$ (then, $|R_i| \le \frac nt +1$), we assign value $i$ to the variable $x_j$. Let $\sigma_\pi$ be the resulting variable assignment. 

Given $\pi$, let 
\[
E_1 = \left\{e \mid e \in E \mbox{ and } \nexists j \in e \setminus \{\pi(e)\}, \exists i: \{\pi(e),j\} \subseteq R_i\right\},
\]
and let $E_2 = E \setminus E_1$.  In other words, $E_1$ is the set of hyperedges whose winner with $\pi$ lies alone in some $R_i$, and $E_2$ is the set of hyperedges $e$ whose winner with $\pi$ lies in some $R_i$ together with some other element(s) of $e$. We now bound the cardinality of $E_2$:
\begin{align*}
|E_2|&\le\sum_{a=2}^k \sum_{i=1}^t \left(\binom{|R_i|}a \cdot \binom{n-|R_i|}{k-a}\right) \\
&\le\sum_{a=2}^k  \sum_{i=1}^t  \frac{|R_i|^a \cdot (n - |R_i|)^{k-a}}{a!\cdot (k-a)!}\\
&\le\sum_{a=2}^k \sum_{i=1}^t  \frac{\left(\frac nt +1\right)^a \cdot n^{k-a}}{a!\cdot (k-a)!}\enspace\le\enspace\sum_{a=2}^k  \sum_{i=1}^t  O\left(\epsilon^{a} \cdot n^{k}\right)\\
&=
\sum_{a=2}^k  O\left(\epsilon^{a-1} \cdot n^{k}\right)\enspace\le\enspace O\left(\epsilon \cdot n^{k}\right)
\enspace =\enspace c' \cdot \epsilon \cdot n^k,
\end{align*}
for some constant $c' > 0$.
Now, for each $e \in E_1$,  the constraint $C_{\pi(e),e}$  will be satisfied, and no other $C_{i_j,e}$, $i_j \in e \setminus \{\pi(e)\}$, will be satisfied. Thus, if $e \in E_1$, then the constraints in $\Gamma_e$ will contribute $\tau - w_e(\pi(e))$ to the  value of the $k$-CSP solution $\sigma_\pi$.
Moreover, if $e \in E_2$, then the constraints in $\Gamma_e$ will not contribute to the value of the $k$-CSP solution $\sigma_\pi$. Then, the value $v(\sigma_\pi)$ of the $k$-CSP solution $\sigma_\pi$ can be lower bounded as
\begin{align*}
v(\sigma_\pi) &= \sum_{e \in E_1} (\tau - w_e(\pi(e)))\\
&\ge \sum_{e \in E_1} (\tau - w_e(\pi(e))) + \sum_{e \in E_2} (\tau - w_e(\pi(e))) - \tau  |E_2|\\
&=  \sum_{e \in E} (\tau - w_e(\pi(e))) - \tau \cdot |E_2|\\
&\ge |E| \cdot \tau - C(\pi) - c' \cdot \tau \cdot \epsilon \cdot n^k.
\end{align*}
Moreover, if $\sigma$ is any assignment to the $k$-CSP variables, let $\pi_\sigma$ be any permutation that ranks the elements decreasingly by assignment value (i.e., $i \prec_{\pi_{\sigma}} j$ if $x_i < x_j$) breaking ties arbitrarily. Let 
$E'_1 = \left\{ e \in E \mid \exists i \in e \, \forall j \in e \setminus \{i\} : x_i > x_j\right\}$. %
Then, if $e \in E'_1$, it holds that the hyperedge $e$ contributes a value of $w_e(\pi_{\sigma}(e))$ to $C(\pi_{\sigma})$. Note that only the constraints of hyperedges in $E'_1$ contribute to $v(\sigma)$. In particular, for $e\in E'_1$, the constraints $\Gamma_e$ contribute a value of $\tau - w_e(\pi_\sigma(e))$ to $v(\sigma)$. Then,
\begin{align*}
|E| \cdot \tau - C(\pi_{\sigma}) &=
\sum_{e \in E} (\tau - w_e(\pi_\sigma(e)))\\
&\ge \sum_{e \in E'_1} (\tau - w_e(\pi_\sigma(e)))
\enspace
= 
\enspace
v(\sigma).
\end{align*}
Then, $C(\pi_{\sigma}) \le |E| \cdot \tau - v(\sigma)$.

Let $\pi^{\star} = \arg\min_{\pi} C(\pi)$ be a permutation that minimizes the WFHS cost. We have proved that
$v\left(\sigma_{\pi^{\star}}\right) \ge |E| \cdot \tau - C(\pi^{\star}) - c' \cdot \tau \cdot \epsilon \cdot n^k$,
i.e.,
$$C(\pi^{\star}) \ge  |E| \cdot \tau - v(\sigma_{\pi^{\star}}) - c' \cdot \tau\cdot  \epsilon \cdot n^k.$$
Now, let $\sigma^{\star} = \arg\max_{\sigma} v(\sigma)$. By $v(\sigma^{\star}) \ge v(\sigma_{\pi^{\star}})$, we have
$$C(\pi^{\star}) \ge  |E| \cdot \tau - v(\sigma^{\star}) - c' \cdot \tau\cdot  \epsilon \cdot n^k.$$
The algorithm in Theorem~\ref{thm:schudy} of \citet{schudy12}, when run on the max-$k$-CSP instance $\mathcal{I}$ returns (in time %
$O(n^k)$)   an assignment $\tilde{\sigma}$ such that $\ex\left[v(\tilde{\sigma})\right] \ge v(\sigma^{\star}) - \epsilon \cdot \tau \cdot n^k$. Then,
\begin{align*}
C(\pi^{\star}) &\ge  |E| \cdot \tau - v(\sigma^{\star}) - c' \cdot \tau\cdot  \epsilon \cdot n^k\\
&\ge  |E| \cdot \tau - \ex\left[v(\tilde{\sigma})\right] - (c'+1) \cdot \tau\cdot  \epsilon \cdot n^k.
\end{align*}
On the other hand, the permutation $\pi_{\tilde{\sigma}}$ will satisfy $\ex\left[C(\pi_{\tilde{\sigma}})\right] \le |E| \cdot \tau - \ex\left[v(\tilde{\sigma})\right]$. Thus,
$$C(\pi^{\star}) \ge  \ex\left[C(\pi_{\tilde{\sigma}})\right] - (c'+1) \cdot \tau\cdot  \epsilon \cdot n^k,$$
or equivalently,
$$\ex\left[C(\pi_{\tilde{\sigma}})\right] \le C(\pi^{\star}) + (c'+1) \cdot \tau \cdot \epsilon \cdot n^k.$$
The algorithm then returns an expected additive $(c'+1) \cdot \tau \cdot \epsilon \cdot n^k$ approximation. Markov's inequality ensures that, if we run the algorithm $r = O(\alpha \log n)$ times, then with probability at least $1 - n^{-\alpha}$, the best of the $r$ returned solutions is a $\left(c \cdot \tau \cdot \epsilon \cdot n^k\right)$-additive approximation for $c = 2(c'+1)$. 

Finally, given that $c$ is a constant and that $\epsilon$ can be chosen arbitrarily,  we can substitute the value of $\epsilon$ with $\frac{\epsilon}c$ to guarantee that the running time will still be not larger than %
$O(\alpha n^k \log n)$
and that the algorithm returns a $(\tau \cdot \epsilon \cdot n^k)$-additive approximation with probability $1 - n^{-\alpha}$.
\end{proof}
In our application, we will have $\tau = \Theta(n^{-k})$, so that the additive error will  be  as small as $\epsilon$ for any constant $\epsilon > 0$.

\section{RECONSTRUCTING RUMS FROM $k$-SLATES}\label{sec:poly_alg}
Using Theorem~\ref{thm:wfhs}, we give next an approximate separation oracle for $\overline{F}_{\rho}$ when $\mathcal{S}$ is the class of all slates of size $k$. Recall that $\overline{F}_{\rho}$ is defined by~(\ref{eqn:feas}).
\begin{theorem}\label{thm:apprx_sep_oracle}
Let $k$ be a constant and let $\mathcal{S} = \binom{[n]}k$.
Fix any constants $\alpha > 0$ and $0 < \epsilon < \frac12$.
Then, there exists a randomized algorithm such that, given as input an assignment
$\{D\} \cup \{\overline{\Delta}_{S,i}\}_{i \in S \in \mathcal{S}}$ to $\overline{F}_{\rho}$, in time $n^{O(k)}$ and with probability at least $1-n^{-\alpha}$: 
(i)  if at least one of the constraints $c_{\rho}$ or $c_{S,i}$ (for $i \in S \in \mathcal{S}$) is unsatisfied, it returns an unsatisfied constraint of $\overline{F}_{\rho}$; or,
(ii) if there exists at least one $\pi \in \mathbf{S}_n$ such that $\sum\limits_{S \in \mathcal{S}} \overline{\Delta}_{S,\pi(S)} < D+1-2\epsilon$, it returns an unsatisfied constraint of type $c_{\pi}$;  otherwise, 
(iii) it might not return any  unsatisfied constraint (even if some exists).
\end{theorem}

\begin{proof}
First, we check the validity of every constraint of type $c_{\rho}$ or $c_{S,i}$ (for $i \in S \in \mathcal{S}$); this can be done in time $O(n^{k+1})$. If any one of these constraints is unsatisfied, one of them is returned.

Otherwise, we run the algorithm of Theorem~\ref{thm:wfhs} on the WFHS instance given by the $\{\overline{\Delta}_{S,i}\}_{i \in S \in \mathcal{S}}$. This instance is $(\tau,k)$-bounded with $\tau = 2|\mathcal{S}|^{-1}$. Thus, the algorithm of Theorem~\ref{thm:wfhs} returns,  with probability at least $1 - n^{-\alpha}$, an $\left(\epsilon \cdot \tau \cdot |\mathcal{S}|\right)$-additive approximation, which is a $(2\epsilon)$-additive approximation, in polynomial time. In particular, with that probability, it returns a permutation $\tilde{\pi}$ such that $\sum_{S \in \mathcal{S}} \overline{\Delta}_{S,\tilde{\pi}(S)} \le \left(\min_{\pi} \sum_{S \in \mathcal{S}} \overline{\Delta}_{S,\pi(S)}\right) + 2\epsilon$.

The algorithm then checks the validity of the constraint $c_{\tilde{\pi}}$: if it is violated, it returns $c_{\tilde{\pi}}$ and none otherwise.

Observe that if $\min_{\pi} \sum_{S \in \mathcal{S}} \overline{\Delta}_{S,\pi(S)} < D + 1 - 2\epsilon$, the $\{\overline{\Delta}\}_{i \in S \in \mathcal{S}}$ assignment violates  $c_{\tilde{\pi}}$, which will then be returned.  Otherwise if $\min_{\pi} \sum_{S \in \mathcal{S}} \overline{\Delta}_{S,\pi(S)} \ge D + 1 - 2\epsilon$, the $c_{\tilde{\pi}}$ constraint might or might not be violated.
\end{proof}

Theorem~\ref{thm:apprx_sep_oracle} provides an approximate separation oracle for $\overline{F}_{\rho}$. 
If we plug it into the ellipsoid algorithm~\citep{gls88}, we obtain a polynomial-time algorithm to compute a $\delta$-additive approximation to our RUM fitting problem.
\begin{theorem}\label{thm:avg}
Let $k > 0$ be any integer and let $\delta$ be a constant in $0 < \delta <  1$. Then, Problem~\ref{prob:avg} can be approximated to within an additive value
of $\delta$ in time $n^{O(k)}$, where $n$ is the number of elements of the RUM.
\end{theorem}
\begin{proof}
We describe \ipalg, a polynomial-time algorithm with the stated complexity that uses the ellipsoid algorithm of \citet{gls88} as a subroutine. As it is well-known, the ellipsoid algorithm uses a separation oracle as a black-box. In our case, the black-box is the approximate separation oracle of Theorem~\ref{thm:apprx_sep_oracle}.

First, \ipalg guesses $\rho \in \left\{i \cdot \delta/2 \mid 0 \le i \le \left\lceil\frac4\delta\right\rceil\right\}$ (the algorithm performs a binary search among the values in this set). For a given $\rho$, the ellipsoid algorithm is invoked together with the approximate separation oracle of Theorem~\ref{thm:apprx_sep_oracle} to approximately check the non-emptiness of $\overline{F}_{\rho}$. In particular, the ellipsoid algorithm will call the separation oracle at most $n^{O(k)}$ many times\footnote{In our separation oracle, we set $\alpha = c\cdot k$ for some constant $c$; this guarantees that each call to the separation oracle will have the approximation properties of Theorem~\ref{thm:apprx_sep_oracle} with high probability.} returning at most polynomially many separating hyperplanes.  If such a set defines an infeasible LP, the ellipsoid algorithm correctly concludes that $\overline{F}_{\rho}$ is empty. 

Otherwise,  the ellipsoid algorithm  returns a point $x = (D) \,.\, (\overline{\Delta}_{S,i})_{i \in S \in \mathcal{S}}$ that the oracle was unable to separate from $\overline{F}_{\rho}$.
This point could lie inside of $\overline{F}_{\rho}$, or outside of it, since the oracle only guarantees that the constraints of type $c_{\pi}$ hold to within an additive error of twice $\epsilon \triangleq \delta/4$. %
Let us define the point $x' = (D - \delta/2) \,.\, (\overline{\Delta}_{S,i})_{i \in S \in \mathcal{S}}$.
We prove that $x' \in \overline{F}_{\rho - \delta/2}$. Indeed, each $c_{\pi}$ constraint is satisfied by $x'$ ($c_{\pi}$ is off by at most $2\epsilon=\delta/2$ with the solution $x$, and the RHS of the $c_{\pi}$ constraint decreases by $\delta/2$ when switching from $x$ to $x'$). Moreover,   the $c_{\rho-\delta/2}$ constraint of $\overline{F}_{\rho - \delta/2}$ is satisfied by $x'$ ($c_{\rho}$ is satisfied by $x$, thus $c_{\rho -\delta/2}$ is satisfied by $x'$). Each remaining constraint is also satisfied.

Let $i^{\star}$ be the largest $i$ for which the algorithm establishes that $x' \in \overline{F}_{\rho^{\star} - \delta/2}$ where $\rho^{\star} = i^{\star} \cdot \delta/2$.\footnote{Such an $i^{\star}$ exists since the maximum $\ell_1$-distance between two probability distributions is $2 \le \lceil \frac4\delta\rceil \cdot \frac{\delta}2$.}
Then, the dual LP~(\ref{eqn:dual}) does not admit a solution of value at least  $\rho^{\star}+\delta/2$, but admits a solution of value at least $\rho^{\star}-\delta/2$. It follows that the optimal solution of the dual LP~(\ref{eqn:dual}), and thus of the primal LP~(\ref{eqn:prelp}), lies in $[\rho^{\star}-\delta/2,\rho^{\star}+\delta/2]$.
Hence, the ellipsoid algorithm with the above separation oracle, returns a solution that approximates the optimal solution of Problem~\ref{prob:avg} to within $\delta$. 

Finally, to recover an approximating RUM whose average-distance error is at most the smallest possible plus $\delta$, 
\ipalg acts as follows. Consider the run of the ellipsoid algorithm with $\rho = \rho^{\star}$. In this run,  the ellipsoid algorithm calls the separation oracle at most polynomially many times and returns no more than polynomially many separating hyperplanes. Some of these hyperplanes might refer to non-permutation constraints, and the rest refer to  the permutation constraints of, say, permutations $\pi_1, \ldots, \pi_t$ (for $t \le n^{O(k)}$). By restricting the primal LP (\ref{eqn:prelp}) to  its non-permutation variables and to the permutation variables $p_{\pi_1}, \ldots, p_{\pi_t}$, we obtain an LP of size $n^{O(k)}$ (i.e., solvable in  time $n^{O(k)}$), and whose optimal value is at most $\delta$ plus the optimum of the primal LP (\ref{eqn:prelp}). Thus, solving the restricted LP allows us to obtain a RUM with an error no greater than the smallest possible plus $\delta$.
\end{proof}

\paragraph{Succinct Representation.}
As shown in~\citet{ckt21}, every RUM can be sketched to $O(\epsilon^{-2} \cdot k \cdot n \log^2 n)$ bits in such a way  that the probability distribution of {\em each} slate of size at most $k$ is approximated to within an $\ell_1$-error of $\epsilon$. This sketch is a RUM over $O(\epsilon^{-2} \cdot k \cdot \log n)$ permutations.  Consequently, the approximately optimal RUM returned by the algorithm of Theorem~\ref{thm:avg} can be reduced (with the same $O(\delta)$ additive error with respect to the optimal approximating RUM) to a RUM supported by $O(\delta^{-2} \cdot k \cdot \log n)$ permutations.

\begin{algorithm}[ht]
\small
\begin{algorithmic}[1]
\STATE $P \leftarrow \varnothing$
\STATE $\pi^{\star} \leftarrow$ any permutation from $\mathbf{S}_n$
\REPEAT
\STATE $P \leftarrow P \cup \{\pi^{\star}\}$
\STATE Solve the primal LP (\ref{eqn:prelp}) restricted to the variables $\epsilon_{S,i}$ for $i \in S \in \mathcal{S}$, and $p_\pi$ for $\pi \in P$; let $\mathcal{P}$ be its optimal primal solution, and $\mathcal{D}$ be its optimal dual solution (i.e., the solution of LP (\ref{eqn:dual}))
\STATE $\pi^{\star} \leftarrow$ \ViolatedHyperplane($\mathcal{D}$)
\UNTIL{$\pi^{\star} = \bot$}
\RETURN the RUM that samples $\pi \in P$ with probability $\mathcal{P}(p_\pi)$ and $\pi \in \mathbf{S}_n \setminus P$ with probability 0.
\end{algorithmic}
\caption{A heuristic for Problem~\ref{prob:avg} \halg
\citep{ackpt22}.
\label{alg:heuristic}}
\end{algorithm}

\begin{algorithm}[h]
\small
\begin{algorithmic}[1]
\STATE For $\pi\in \mathbf{S}_n$, let $\wfhs(\pi)  =  \sum_{S\in\mathcal{S}} \mathcal{D}(\Delta_{S,\pi(S)})$ 
\STATE For $\pi\in \mathbf{S}_n$, let $N(\pi)$ be the set of permutations that can be obtained from $\pi$ by moving one of its elements %
\STATE let $0<t'\leq t$ be two integers
\STATE $\wfhs_{\min} \leftarrow \infty$
\FOR{$i = 1, \ldots, t$}
\STATE $\pi \gets$ uniform at random permutation from $\mathbf{S}_n$
\WHILE{$\exists \pi' \in N(\pi)$ such that $\wfhs(\pi') < \wfhs(\pi)$}
\STATE $\pi \leftarrow \arg\min_{\pi' \in N(\pi)} \wfhs(\pi')$
\ENDWHILE
\IF{$\wfhs(\pi) < \wfhs_{\min}$}
\STATE $\wfhs_{\min} \leftarrow \wfhs(\pi)$
\STATE $\pi_{\min} \leftarrow \pi$
\ENDIF
\IF{$\wfhs_{\min} < \mathcal{D}(D)$ \AND $i\ge t'$}
\RETURN $\pi_{\min}$
\ENDIF
\ENDFOR
\RETURN $\bot$
\end{algorithmic}
\caption{\label{alg:localsearch}A randomized local-search for \ViolatedHyperplane. In experiments, we set $t=100$ and $t'=5$.}
\end{algorithm}

\section{A NON-ADAPTIVE LOWER BOUND}\label{sec:nalb}

Our reconstruction algorithm leverages on knowing the winning distribution of each slate of size $k$. We show here that a non-adaptive algorithm that aims to approximate the winning distribution of each slate of size $k$, must access a constant fraction of slates  of size $k$ during learning.
\begin{theorem}\label{thm:nalb}
Let $\mathcal{A}$ be a non-adaptive algorithm that only queries an $\epsilon$ fraction of the slates of size $k \ge 2$. Then, with probability at least $1-\epsilon$, the expected $\ell_1$-error of $\mathcal{A}$'s prediction on at least one slate of size $k$ is at least $2 - \frac2k$.
\end{theorem}
\begin{proof}
This result can be proved with a very simple RUM. Let $S$ be a slate sampled uniformly at random from the class $\binom{[n]}k$, and let $i$ be sampled uniformly at random from the slate $S$. The RUM $R = R^{i,S}$ will be supported by a single permutation $\pi_{i, S}$ that has the element of $[n] \setminus S$ in its first $n - |S|$ positions (sorted arbitrarily), element $i$ in its $(n-|S|+1)$st position, and the element of $S \setminus \{i\}$ in its last $|S|-1$ positions (again, sorted arbitrarily).

Observe that $R_S(i) = 1$ and $R_S(j) = 0$, for each $j \in S \setminus  \{i\}$. Moreover, if one queries $R$ on any slate in $\binom{[n]}k \setminus \{S\}$, one is unable to tell which element of $S$ ranks highest in $\pi$ and, thus, in permutations sampled from $R$. Indeed, the elements of $S$ lie in the last $k$ positions of $\pi$ and thus no slate of size $k$ other than $S$ will  result in some element of $S$ winning.

Thus, a non-adaptive algorithm that did not query $S$ in its learning phase, can correctly guess that one element $i'$ of $S$ will always win in $S$.  However, from the algorithm's perspective, $\Pr[j = i'] = \frac1{|S|}$ for each $j \in S$. Thus, if the goal of the algorithm is to minimize the expected $\ell_1$-error on its guess for $D_S$, it should return the uniform vector $\left(\frac1{|S|},\ldots,\frac1{|S|}\right)$, since it is the (geometric) median of the $k$ possible distributions for $D_S$, and since each of these distributions is equally likely. The expected $\ell_1$-error on $D_S$  of any non-adaptive algorithm that does {\em not} query $S$ during its learning phase is then at least $1 \cdot \left(1 - \frac{1}{|S|}\right) + \left(|S|-1\right) \cdot \left(\frac1{|S|} - 0\right) = 2 - \frac{2}{|S|} = 2 - \frac2k$.

Moreover, if the non-adaptive algorithm queries only $\epsilon \cdot \binom n k$ slates during its learning phase, then it will query $S$ with probability at most $\epsilon$. If this event does not happen, the algorithm's expected $\ell_1$-error on $S$ is $\geq 2 \cdot \left(1 - \frac1k\right)$.
\end{proof}

\ifallproofs
\fi

\begin{table}[h!]
{\small
\centering
\begin{tabular}{@{}lr|rr|ll|ll@{}}\toprule
\multirow{3}{*}{Dataset} & \multirow{3}{*}{$n$} & \multirow{3}{*}{$k$} & \multirow{3}{*}{$|P|$} & \multirow{2}{*}{average} & \multirow{2}{*}{lower}  & MNL  & MNL \\ 
&  &  &  &  \multirow{2}{*}{error} & \multirow{2}{*}{bound} & (class.) & (d.c.) \\
&  &  &  &  &  & avg. err. & avg. err. \\
\hline
\multirow{4}{*}{\Sushi} & \multirow{4}{*}{10} & 2 & 46 & 0 & 0 & $2\cdot 10^{-5}$ & 0.0543\\ 
 &  & 3 & 241 & 0 & 0 & 0.0389 & 0.0691 \\
 &  & 4 & 631 & 0 & 0 & 0.0354 & 0.0777\\
 &  & 5 & 916 & 0.0002 & 0 & 0.0282 & 0.0798\\ \hline
\multirow{4}{*}{\SFwork} & \multirow{4}{*}{6} & 2 & 16 & 0 & 0 & $2\cdot 10^{-5}$ & 0.0595 \\
 & & 3 & 36 & 0.0044 & 0.0044 & 0.0317 & 0.0756 \\
 & & 4 & 35 & 0.0072 & 0.0072 & 0.0235 & 0.0597 \\
 & & 5 & 20 & 0.0035 & 0.0035 & $2\cdot 10^{-5}$ & 0.0338 \\ \hline
\multirow{4}{*}{\SFshop} & \multirow{4}{*}{8} & 2 & 29 & 0 & 0 & $2\cdot 10^{-5}$ & 0.0186\\ 
 & & 3 & 109 & 0.0002 & 0.0002 & 0.0493 & 0.0314 \\
 & & 4 & 195	& 0.0010 & 0.0010 & 0.0348 & 0.0352\\
 & & 5 & 192	& 0.0017 & 0.0017 & 0.0203 & 0.0321\\ \hline 
\multirow{4}{*}{\Afive} & \multirow{4}{*}{16} & 2 & 121 & 0 & 0 & $1 \cdot 10^{-5}$ & 0.0473 \\ 
 & & 3 & 951 & 0.0005 & 0 & 0.0443 & 0.0666 \\
 & & 4 & 1044 & 0.0105 & 0 & 0.0497 & 0.0811 \\
 & & 5 & 1001 & 0.0207 & 0.0078 & 0.0542 & 0.0951 \\ \hline  
\multirow{4}{*}{\Anine} & \multirow{4}{*}{12} & 2 & 67 & 0 & 0 & $1 \cdot 10^{-5}$ & 0.0465 \\ 
& & 3 & 441 & 0 & 0 & 0.0318 & 0.0626 \\
& & 4 & 1053 & 0.0014 & 0 & 0.0347 & 0.0728 \\
& & 5 & 1032 & 0.0080 & 0 & 0.0387 & 0.0844 \\ \hline
\multirow{4}{*}{\Aseventeen} & \multirow{4}{*}{13} & 2 & 79 & 0 & 0 & $3 \cdot 10^{-5}$ & 0.0527 \\
& & 3 & 573 & 0 & 0 & 0.0441 & 0.0800\\
& & 4 & 829	& 0.0127 & 0.0038 & 0.0603 & 0.1051 \\
& & 5 & 843	& 0.0281 & 0.0253 & 0.0732 & 0.1309\\ \hline
\multirow{4}{*}{\Afortyeight} & \multirow{4}{*}{10} & 2 & 46 & 0 & 0 & $1 \cdot 10^{-5}$ & 0.0384\\
& & 3 & 241 & 0 & 0 & 0.0319 & 0.0630\\
& & 4 & 561 & 0.0026 & 0.0002 & 0.0392 & 0.0891\\
& & 5 & 433 & 0.0224 & 0.0208 & 0.0507 & 0.1155\\ \hline
\multirow{4}{*}{\Aeightyone} & \multirow{4}{*}{11} & 2 & 56 & 0 & 0 & $2 \cdot 10^{-5}$ & 0.0520\\ 
& & 3 & 325 & 0.0005 & 0.0005 & 0.0484 & 0.0944\\
& & 4 & 513 & 0.0250 & 0.0233 & 0.0659 & 0.1255\\
& & 5 & 425 & 0.0535 & 0.0521 & 0.0859 & 0.1540\\ \bottomrule
\end{tabular}
\caption{Results of the fitting experiments. $|P|$ is the size of the support found by \halg. %
For some datasets, we found a non-trivial lower bound on the average error achievable via a RUM. The last two columns represent, respectively, the average $\ell_1$-error of the (classifier) MNL model and the classical discrete choice MNL model. \label{tab:fitting}}
}
\end{table}

\section{EXPERIMENTS}\label{sec:exp}

We perform two types of experiments\footnote{code available at \href{https://github.com/mirkogiacchini/k-wise-RUMs}{\url{https://github.com/mirkogiacchini/k-wise-RUMs}}}. In Section \ref{sec:exp-rum-fit} we are given as input a set of slates and for each slate a winner, chosen by a user, over its elements. We evaluate the quality of our algorithm in representing the resulting winner distributions by measuring the $\ell_1$-error between the winner distributions induced by the dataset and those given by the RUM found by the algorithm. 
Next, in Section \ref{sec:exp-gen}, we consider a prediction setting where the data is split into training and test set. We learn a RUM using our algorithm on the training set and evaluate its generalization quality on the test set.

\subsection{Experimental Setup}

For practical reasons we did not implement the polynomial-time algorithm \ipalg, but based on similar ideas, used the heuristic \halg described in Algorithm~\ref{alg:heuristic}, which is from~\citet{ackpt22}. The heuristic needs access to a separation oracle (\ViolatedHyperplane) that, as shown in Section \ref{sec:poly_alg}, can be seen as an instance of the WFHS problem. We implemented such an oracle both with an exact algorithm running in time $O(n^k2^n)$ (which we describe in the Appendix, and that is a generalized version of the exact algorithm of \citet{l64} for the FAS problem) and also using a randomized local-search heuristic described in Algorithm \ref{alg:localsearch}; the latter is a generalization of the algorithm in~\citet{ackpt22}. We were able to use the exact oracle on most of the datasets, but when it turned out to be too slow, we resorted to the local-search (in particular, we needed the local-search on dataset \Afive). Note that, when provided with an exact oracle, \halg is guaranteed to find the optimal RUM, however, we have no (non-trivial) upper bound on its running time. On the other hand, when we use the local-search heuristic, \halg loses also its guarantee to converge to the optimal RUM, since the local-search might fail at finding a separating hyperplane even if it exists. 
We implemented \halg in Python using IBM cplex\footnote{\url{https://www.ibm.com/analytics/cplex-optimizer}} and we ran it on general-purpose hardware\footnote{Intel core i7, 8GB of RAM}.

\par{\emph{Baselines.}} We used two baselines both in the fitting experiments and in the prediction ones. The first one is the classical MNL model used in discrete choice \citep{btl}. We used the implementation of \citet{ru16}. The second baseline casts the problem as a classification task and uses the classical multinomial logit classifier. We used the scikit-learn\footnote{\url{https://scikit-learn.org/stable/}} implementation of this model. Note that the second model is not a RUM. 

\par{\emph{Slate sizes.}}
We ran the experiments on sets of equal-sized slates, in particular, we considered slates of size between 2 and 5, i.e., $\mathcal{S}\subseteq\binom{[n]}{k}$ for $k\in\{2,3,4,5\}$. Note that we allowed missing $k$-slates (i.e., possibly  $\mathcal{S}\subsetneq\binom{[n]}{k}$). 

\par{\emph{Datasets.}} We performed the experiments on the following:

(i) \Sushi dataset \citep{k03} contains a list of permutations over 10 elements, representing people's preferences over different types of sushi. Since we only care about $k$-slates, we transformed each permutation into $\binom{n}{k}$ different $k$-slates setting the winner according to the permutation. Note that a RUM with error 0 exists on such a dataset.

(ii) Datasets \SFwork and \SFshop \citep{kb06} contain a list of choices between different transportation alternatives made by people going to work or to a shopping center, respectively. Since these datasets contain only a few slates of fixed size, we augmented them in the following way: each slate $S$, $|S|\geq k$, with winner $w\in S$, is transformed into $\binom{|S|-1}{k-1}$ slates of size $k$, each having $w$ as a winner. This transformation might induce a bias, however, it seems reasonable in practice. 

(iii) Datasets \Afive, \Anine, \Aseventeen, \Afortyeight, \Aeightyone \citep{election} contain lists of election ballots, which are partial ordering of the elements (i.e., each ballot is a sorted subset of $S$). As in \Sushi, we converted the sorted subsets into several $k$-slates assigning the winner according to the ordering.

\begin{table}[ht]
{\small
\centering
\begin{tabular}{@{}lr|r|c|ll|l@{}}\toprule
\multirow{2}{*}{Dataset} & \multirow{2}{*}{$n$} & \multirow{2}{*}{$k$} & \multirow{2}{*}{\halg} & MNL & MNL & Train \\ 
&  &  &  & (class.) & (d.c.) & Tensor   \\ \hline
\multirow{4}{*}{\Sushi} & \multirow{4}{*}{10} & 2 & 0.023 & 0.023 & 0.054 & 0.023 \\ 
&  & 3 & 0.027 & 0.037 & 0.056 & 0.027 \\
&  & 4 & 0.028 & 0.033 & 0.054 & 0.029 \\
&  & 5 & 0.028 & 0.030 & 0.051 & 0.030 \\ \hline
\multirow{4}{*}{\SFwork} & \multirow{4}{*}{6} & 2 & 0.091 & 0.088 & 0.093 & 0.088 \\ 
&  & 3 & 0.094 & 0.087 & 0.096 & 0.094 \\
&  & 4 & 0.085 & 0.074 & 0.079 & 0.081 \\
&  & 5 & 0.071 & 0.072 & 0.064 & 0.072 \\ \hline
\multirow{4}{*}{\SFshop} & \multirow{4}{*}{8} & 2 & 0.081 & 0.080 & 0.070 & 0.081 \\ 
&  & 3 & 0.066 & 0.067 & 0.058 & 0.066 \\
&  & 4 & 0.060 & 0.057 & 0.055 & 0.062 \\
&  & 5 & 0.056 & 0.051 & 0.051 & 0.058 \\ \hline
\multirow{4}{*}{\Anine} & \multirow{4}{*}{12} & 2 & 0.046 & 0.046 & 0.055 & 0.046 \\ 
&  & 3 & 0.058 & 0.055 & 0.066 & 0.059 \\
&  & 4 & 0.065 & 0.062 & 0.074 & 0.071 \\
&  & 5 & 0.070 & 0.069 & 0.081 & 0.081 \\ \hline
\multirow{4}{*}{\Aseventeen} & \multirow{4}{*}{13} & 2 & 0.107 & 0.106 & 0.100 & 0.107 \\ 
&  & 3 & 0.141 & 0.128 & 0.132 & 0.147 \\
&  & 4 & 0.156 & 0.150 & 0.157 & 0.177 \\
&  & 5 & 0.170 & 0.168 & 0.178 & 0.201 \\ \hline
\multirow{4}{*}{\Afortyeight} & \multirow{4}{*}{10} & 2 & 0.071 & 0.071 & 0.070 & 0.071 \\ 
 &  & 3 & 0.094 & 0.084 & 0.091 & 0.094 \\ 
&  & 4 & 0.105 & 0.097 & 0.109 & 0.114 \\ 
&  & 5 & 0.117 & 0.112 & 0.127 & 0.132 \\ \hline
\multirow{4}{*}{\Aeightyone} & \multirow{4}{*}{11} & 2 & 0.091 & 0.090 & 0.089 & 0.091 \\ 
&  & 3 & 0.121 & 0.113 & 0.124 & 0.126 \\
&  & 4 & 0.143 & 0.138 & 0.151 & 0.158 \\
&  & 5 & 0.168 & 0.166 & 0.180 & 0.193 \\ \bottomrule
\end{tabular}
\caption{Results of the prediction experiments. We run a 5-fold cross validation with 10 different seeds for each dataset and algorithm. The table reports the (avg.) RMSE of each algorithm. The standard deviations are in  $[0.012, 0.025]$ for \SFwork and in  $[0.001, 0.009]$ for others. \label{tab:generalization}}
}
\end{table}

\subsection{RUM Fitting} \label{sec:exp-rum-fit}
The results of the fitting experiments are shown in Table~\ref{tab:fitting}. We let \halg run for a maximum of 1500 iterations and stop it earlier if the average error decreased by less than $10^{-5}$ in 20 iterations (meaning that the algorithm converged). Note that our heuristic always obtains a smaller error than both baselines, and it seems therefore more suitable for representing the datasets. 

Due to the use of the local-search heuristic and to the fact that we stop the algorithm after a fixed number of iterations, \halg will not, in general, converge to the optimal RUM. However, it is possible to find lower bounds on the best possible average error achievable via a RUM. Consider the dual LP (\ref{eqn:dual}) restricted to have only the permutation constraints of $P$, $P\subseteq \mathbf{S}_n$. Consider an optimal solution $\mathcal{D}$ to such restricted LP having value $v$ and suppose we know, via an exact algorithm for WFHS, that $y = \min_{\pi\in\mathbf{S}_n}{\sum_{S\in\mathcal{S}}\mathcal{D}(\Delta_{S, \pi(S)})}$. If $y \geq \mathcal{D}(D)$, then $v$ is the optimal value of the original dual LP (\ref{eqn:dual}), and therefore of the primal LP (\ref{eqn:prelp}). Otherwise, note that the point $(y)\,.\, (\Delta_{S, i})_{i\in S \in \mathcal{S}}$ is a feasible solution to the dual LP (\ref{eqn:dual}) and it has value $x = v - (\mathcal{D}(D) - y)$, therefore $x$ is a lower bound on the optimal value of the primal LP (\ref{eqn:prelp}). \balance Fortunately, this trick to find a lower bound can be done at the end of the fitting procedure, so to run the exact algorithm for WFHS only once. Using these simple remarks we were able to find non-trivial lower bounds on the average error for several of the considered datasets.  In practice, \halg often finds the optimal (or almost-optimal) RUM; furthermore, the error of \halg seems to increase as we increase the size of the slates. We analyzed how the errors are distributed (see Figure \ref{fig:error-distributon}) and it seems that the majority of the slates incur an error below the average, in particular, note that most of the 4-slates have an error close to 0.

\begin{figure}[t]
\begin{center}
\subfigure[$k=4$]{\centerline{\includegraphics[width=7cm]{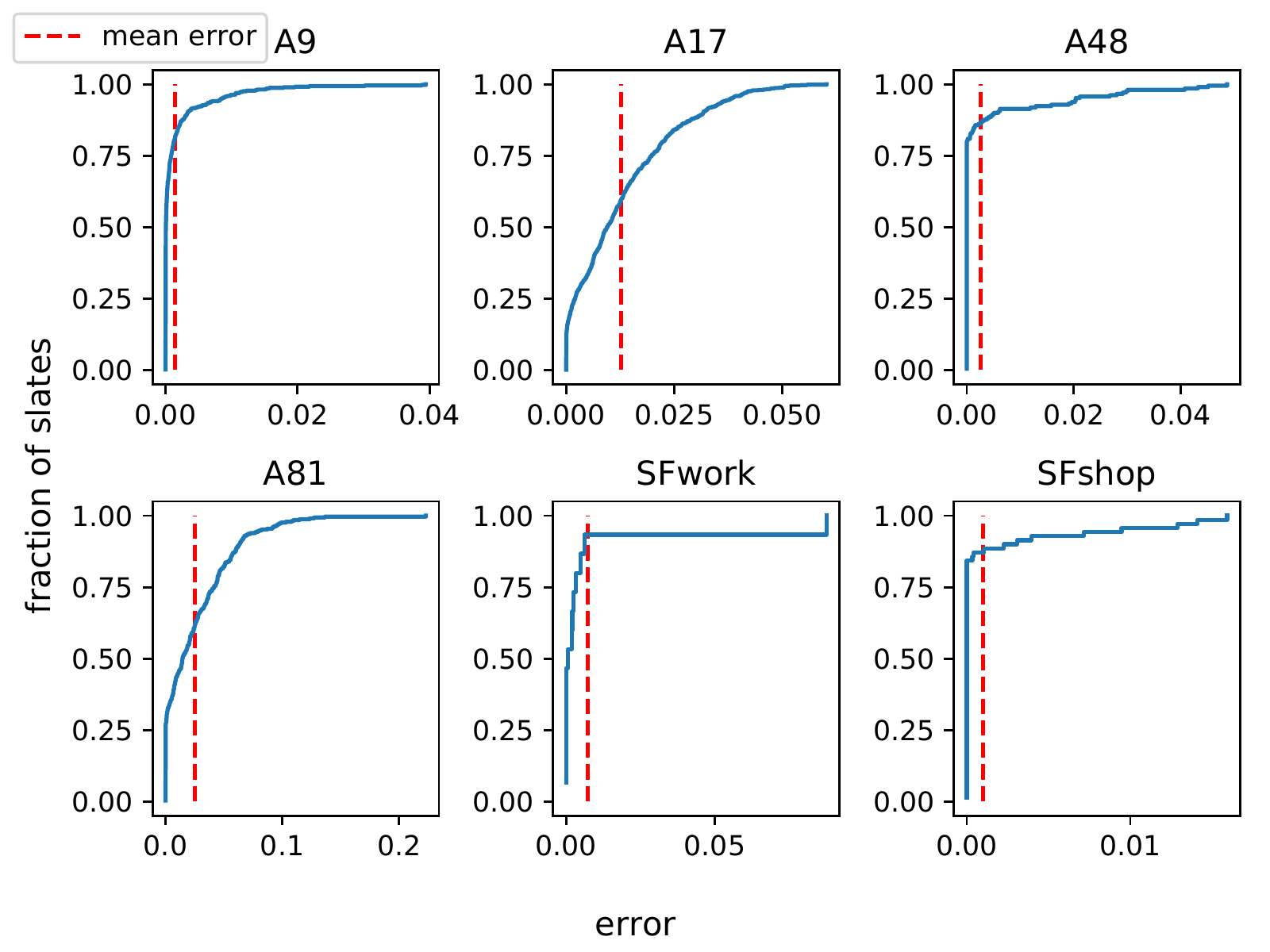}}}
\subfigure[$k=5$]{\centerline{\includegraphics[width=7cm]{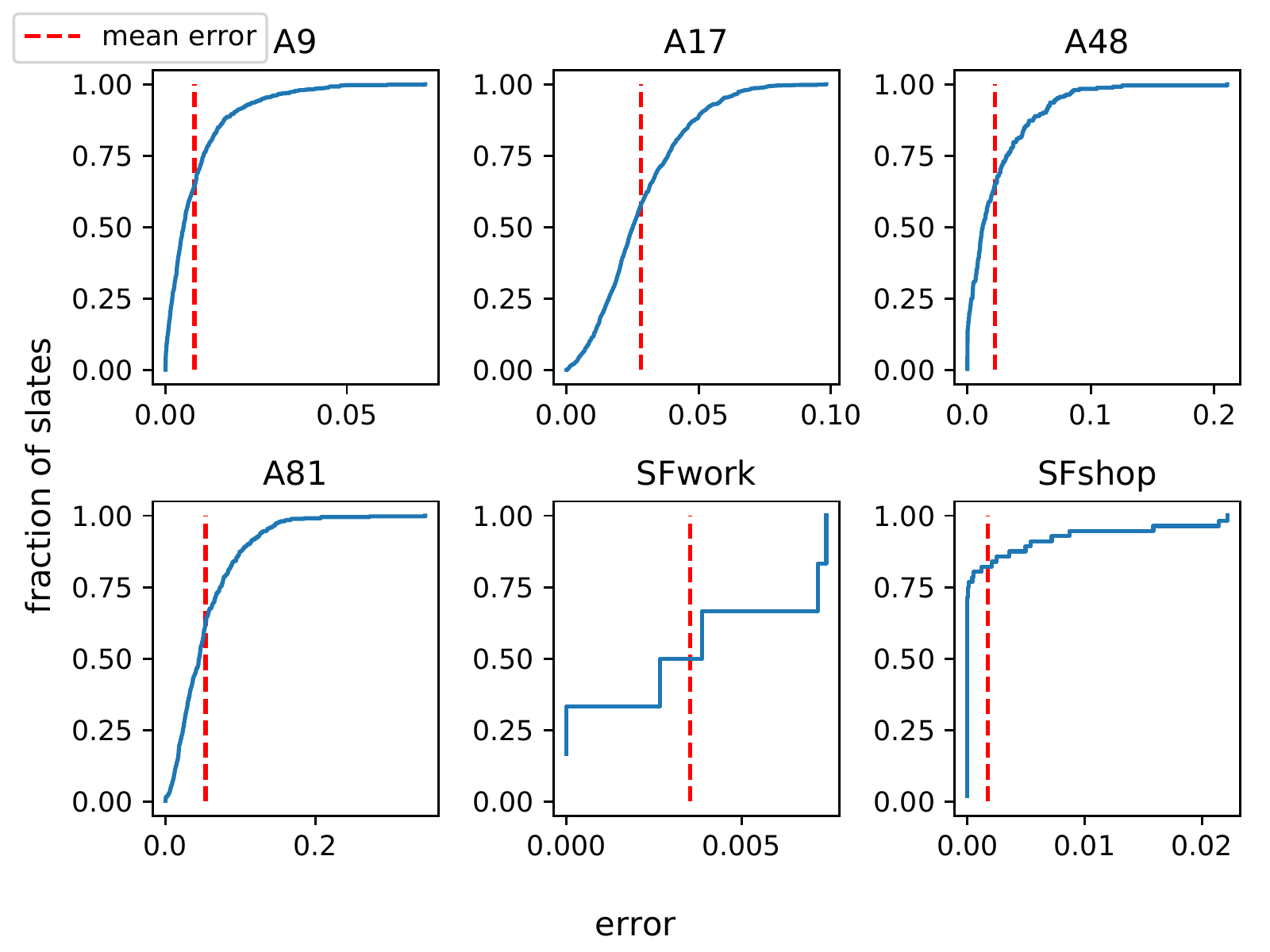}}}
\vskip -0.1in
\caption{Distribution of the errors over the slates. Fixed an error $x$, the corresponding $y$ value is the ratio $\frac{|\{S \mid S\in\mathcal{S} ~\mathrm{ and }~ |R_S - D_S|_1 \leq x\}|}{|\mathcal{S}|}$, where $R$ is the computed RUM and $D_S$ is the empirical distribution over $S$.}
\label{fig:error-distributon}
\vskip -0.4in
\end{center}
\end{figure}

\subsection{Quality of Predictions} \label{sec:exp-gen}
We performed an $\ell$-fold cross-validation to assess the generalization quality of \halg. In particular, given the slates (with repetitions), we divided them into $\ell$ groups of roughly the same size, built at random. We then performed $\ell$ iterations: during the $i$th iteration, the $i$th group is used as the test set and the union of the other $\ell-1$ groups is used as the training set. In our experiments, we set $\ell=5$ and we repeated the random splitting 10 times with different seeds.

As in \citet{ackpt22, mu19}, since the splitting is performed on the initial slates, the same slate might appear both in training and testing. Following them, we used root mean-squared error $\mathrm{RMSE}(D, \hat{D}) = \sqrt{|\mathcal{S}|^{-1}\sum_{i\in S\in\mathcal{S}}{(D_S(i)-\hat{D}_S(i))^2}}$ to evaluate the performances of the algorithms. We trained \halg for a maximum of 250 iterations. The results are shown in Table \ref{tab:generalization}.  (The results for the case $k=2$ differ from \citet{ackpt22} because they normalized by $\binom{n}{2}$ rather than $|\mathcal{S}|$, however, using their definition we get essentially the same results.)  Note that the algorithm gives performances comparable to both baselines; furthermore, using the training data directly to make predictions\footnote{We used the uniform distribution for slates that appear only in the test set.} gives results only slightly worse than \halg, this is due to the fact that \halg can represent the input data with a very small error.

\section{CONCLUSIONS}

In this paper we obtained a polynomial-time algorithm for finding a RUM that best approximates a given set of winning distributions for slates of any constant size $k$.  While our lower bound shows that the reconstruction algorithm is query-optimal, extending it work for adaptive algorithms is an interesting research direction.  Developing provably good algorithms that avoid the ellipsoid method is also an intriguing question that merits further investigation.

\subsubsection*{Acknowledgments}
We thank Pasin Manurangsi for useful discussions and suggestions.
Flavio Chierichetti and Alessandro Panconesi were supported in part by BiCi---Bertinoro International Center for Informatics.  Flavio Chierichetti was supported in part by the Google Gift ``Algorithmic and Learning Problems in Discrete Choice'' and by the PRIN project 2017K7XPAN.

\bibliographystyle{unsrtnat}
\bibliography{main}

\newpage
\onecolumn

\appendix

\section{AN EXACT ALGORITHM FOR WFHS}
In this section we give Algorithm~\ref{alg:wfhs_dp}, a dynamic programming  for the WFHS problem (Problem~\ref{prob:wfhs}). Our Algorithm generalizes the algorithm from~\citet{l64} for the FAS problem.
\begin{algorithm}[ht]
\small
\begin{algorithmic}[1]
\STATE $C(\varnothing) \leftarrow 0$
\FOR{$j = 1, \ldots, n$}
    \FOR{$A \in \binom{[n]}{j}$}
        \FOR{$a \in A$}
            \STATE $t_a\leftarrow 0$
        \ENDFOR
        \FOR{$e \in E$ such that $e\subseteq A$}
            \FOR{$a \in e$}
                \STATE $t_a \leftarrow t_a + w_e(a)$
            \ENDFOR
        \ENDFOR
        \STATE $C(A)\leftarrow \min_{a\in A}\left(t_a + C(A \setminus \{a\})\right)$
        \STATE $\ell(A) \leftarrow a$, for any $a$ such that $t_a + C(A \setminus \{a\}) = C(A)$
    \ENDFOR
\ENDFOR
\STATE $A_n \leftarrow [n]$
\STATE Let $\pi^{\star}$ be an array of size $n$
\FOR{$j=n, \ldots, 1$}
    \STATE $\pi^{\star}_j \leftarrow \ell(A_j)$
    \STATE $A_{j-1}\leftarrow A_j \setminus \{\ell(A_j)\}$
\ENDFOR
\RETURN $\pi^{\star}$
\end{algorithmic}
\caption{An Algorithm for WFHS\label{alg:wfhs_dp}.}
\end{algorithm}
\begin{theorem}
Algorithm~\ref{alg:wfhs_dp} returns an optimal solution to WFHS in time $O\left(k \cdot |E| \cdot  2^n\right) \le O\left(n^k \cdot 2^n\right)$.
\end{theorem}
\begin{proof}

For a given $A \subseteq [n]$, let $C(A)$ be the minimum WFHS cost of a solution to the instance projected on the elements of $A$ (i.e., to the instance obtained by removing each hyperedge that is not fully contained in $A$). Then, $C(\varnothing) = 0$ and 
$$C(A) = 
\min_{a \in A} \left(C(A \setminus \{a\}) + \sum_{\substack{e \in E\\a\in e \subseteq A}} w_e(a)\right).$$

Then, the time required to fill entry $A$ of the array  is at most $O\left(|E| \cdot k\right) = O(n^k)$ --- indeed, we can initialize one variable $t_a = 0$ for each $a \in A$, for a total of $|A| \le n$ variables).  We then iterate over the edges of $E$: for each $a \in e \subseteq A$, we add $w_e(a)$ to $t_a$. The value of $C(A)$ is then the minimum, over $a \in A$, of $C(A\setminus\{a\}) + t_a$; we also set $\ell(A)$ to be the item of $A$ that achieves this minimum

Once array $A$ is filled,  a permutation $\pi^{\star}$ having minimum WFHS cost, $C(\pi^{\star}) = C([n])$,  can  be easily obtained. Let $A_n = [n]$. In general, the element in position $1 \le j \le n$ of $\pi^{\star}$ will be equal to $\ell(A_j)$, and for $1 \le j \le n$, $A_{j-1} = A_{j} \setminus \{\ell(A_{j})\}$. Thus, an extra iteration over the positions $n,n-1,\ldots,1$ is  sufficient to obtain $\pi^{\star}$.

Finally, filling the array takes time $O(k \cdot |E| \cdot 2^n) \le O(k \cdot \binom nk \cdot 2^n) \le O(n^k \cdot 2^n)$ and computing $\pi^{\star}$ takes time $O(n)$.
\end{proof}

\end{document}